\documentclass{article}

\usepackage{listings}
\usepackage{microtype}
\usepackage{graphicx}
\usepackage{subfigure}
\usepackage{booktabs} 
\usepackage{algorithm}
\usepackage{algorithmic}    
\usepackage{multirow}
\usepackage{ bbold }
\usepackage[backref=page]{hyperref}
\usepackage{float}

\usepackage[accepted]{icml2025}

\usepackage{amssymb}
\usepackage{mathtools}
\usepackage{amsthm}

\DeclareMathOperator*{\argmin}{arg\,min} 
\DeclareMathOperator*{\argmax}{arg\,max} 
\DeclareMathOperator*{\tr}{tr}
\DeclareMathOperator*{\diag}{diag}

\usepackage[capitalize,noabbrev]{cleveref}

\theoremstyle{plain}
\newtheorem{theorem}{Theorem}[section]

\theoremstyle{definition}
\newtheorem{definition}[theorem]{Definition}

\theoremstyle{remark}

\usepackage[textsize=tiny]{todonotes}

\icmltitlerunning{Occam's model: Selecting simpler representations for better transferability estimation}

\begin{document}

\twocolumn[
\icmltitle{Occam's model: Selecting simpler representations for better transferability estimation}



\icmlsetsymbol{equal}{*}

\begin{icmlauthorlist}
\icmlauthor{Prabhant Singh}{equal,yyy}
\icmlauthor{Sibylle Hess}{equal,yyy}
\icmlauthor{Joaquin Vanschoren}{equal,yyy}

\end{icmlauthorlist}

\icmlaffiliation{yyy}{Eindhoven University of Technology, Netherlands}

\icmlcorrespondingauthor{Prabhant Singh}{p.singh@tue.nl}

\icmlkeywords{Machine Learning, ICML}

\vskip 0.3in
]




\begin{abstract}
Fine-tuning models that have been pre-trained on large datasets has become a cornerstone of modern machine learning workflows. With the widespread availability of online model repositories, such as Hugging Face, it is now easier than ever to fine-tune pre-trained models for specific tasks. This raises a critical question: which pre-trained model is most suitable for a given task? This problem is called transferability estimation.
In this work, we introduce two novel and effective metrics for estimating the transferability of pre-trained models. Our approach is grounded in viewing transferability as a measure of how easily a pre-trained model's representations can be trained to separate target classes, providing a unique perspective on transferability estimation. We rigorously evaluate the proposed metrics against state-of-the-art alternatives across diverse problem settings, demonstrating their robustness and practical utility. Additionally, we present theoretical insights that explain our metrics' efficacy and adaptability to various scenarios. We experimentally show that our metrics increase Kendall's Tau by up to 32\% compared to the state-of-the-art baselines. 
\end{abstract}

\section{Introduction}\label{introduction}
Using models pre-trained on large datasets like ImageNet~\cite{imagenet} has become a standard practice in real-world deep-learning scenarios. For example, the top five models on HuggingFace have been downloaded more than 200M times. HuggingFace hosts more than 15K models for image classification. The performance and efficiency gain from using models pre-trained on large datasets like ImageNet 21k~\cite{imagenet} and LIAON~\cite{schuhmann2022laionb} is enormous. However, these performance gains can vary considerably depending on model architecture, weights, and the dataset it was pre-trained on (source dataset). This leads to the pre-trained model selection problem. Although the model selection task has strong roots in AutoML, applying classic model selection paradigms is computationally too expensive in this scenario. Fine-tuning each model on the target data for a search strategy like Bayesian optimization is not feasible since the fine-tuning step is too expensive. This raises the question: \textit{"How can we find a high-performing pre-trained model for a given target task without fine-tuning our models and without access to the source dataset."}

\begin{figure}[t]
    \centering
    \includegraphics[width=\columnwidth]{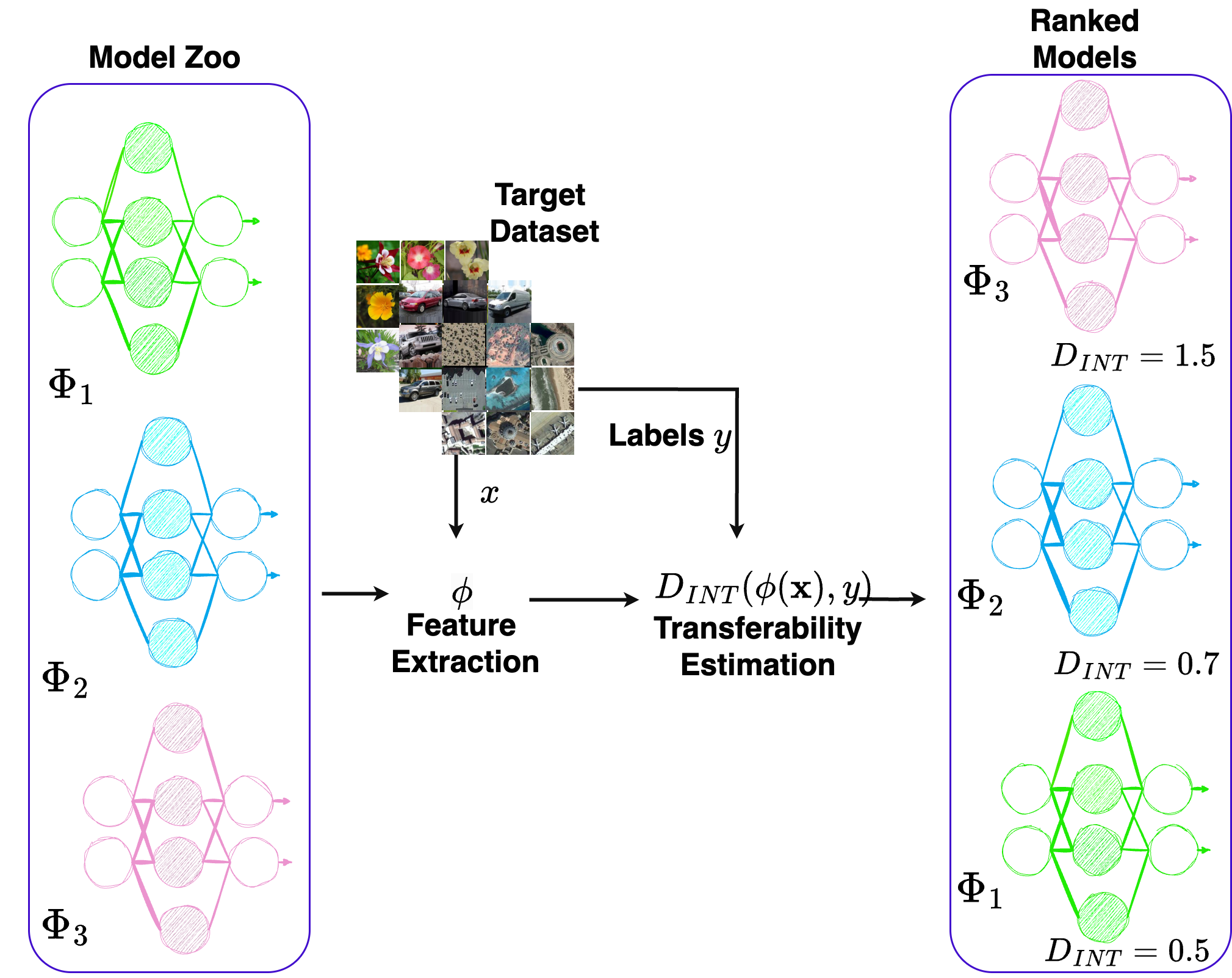}
    \caption{Given a set of pretrained models and a target dataset, extract the embeddings and estimate a complexity score $D_{INT}$ to rank pretrained models based on their suitability for the target task.}
    \label{fig:TransferabilityEstimation}
\end{figure}

This question is answered by transferability estimation methods. The idea behind transferability estimation is to assign a score to each pre-trained model of a given set for a target task, such that one can select the top-performing model for the given task. In the last decade, there have been multiple works, addressing this problem from various perspectives. For example, TransRate~\cite{TransRate} treats the problem from an information theory point of view, and ETran connects the problem of transferability estimation to energy-based models. There have been numerous methods that treat this problem from multiple perspectives like linearization, Bayesian modeling, matrix analysis, etc. However, we found that these approaches fall short in many practical scenarios.

In this work, we introduce a new transferability estimation metrics by analyzing how easily a pre-trained model's representations separate target classes. A well-trained model should produce embeddings where class distinctions are approximately clear, reducing what we call 'representational complexity'. Recent work by \citet{repcomplexity} uses data complexity measures by \citet{hobasu} to evaluate this notion of representational complexity and shows how data complexity evolves through the network and how it changes during training.  We show that we can evaluate representational complexity by assessing interclass separation and concept variance. 


In this work we were inspired by the work in cluster separation and concept variance, leading us to derive two new transferability estimation metrics. We treat the embeddings of the target task from the pre-trained model and true labels to compute our transferability metrics. We examine the embeddings generated for an unseen task by a pre-trained network and assess how challenging the task is for the network to fine-tune. Our first solution treats this problem as a clustering quality problem and the second one treats it as a Concept Variance problem. We compare performance of our metrics with the current state of the art from various families of metrics. Our solutions outperform current state-of-the-art metrics. We also show the effectiveness of our methods in a new challenging search space with modern neural networks and complex datasets in two different settings.  In conclusion, this work
In conclusion in this work:
\begin{itemize}
    \item We introduce two different metrics for transferability estimation, $INT$ and Concept Variance.
    \item We propose a new and challenging experimental benchmark for robust evaluation of transferability estimation metrics in multi-domain(7 domains) image classification settings with modern neural networks. We show that our method outperforms current state-of-the-art metrics with a significant increase in Kendall's tau $\tau_\omega$(32\%). 
    \item We present theory and insight behind the effectiveness of our methods.
\end{itemize}

\section{Related work}\label{related works}
The idea behind transferability estimation is simple: to estimate which model from a zoo would perform best after fine-tuning the model. Transferability estimation as a field is fairly new, the H-Score~\cite{Hscore} and NCE~\cite{NCE} can be considered as early works on this topic, introducing the evaluation of transferability, and the assignment of models corresponding to an estimate of their transferability, for a given target task. 

There are two widely accepted problem scenarios for transferability estimation: source-dependent transferability estimation (where one has access to the source and target dataset) and source-independent transferability estimation (where one does not have access to the source dataset). 
\subsection{Source Dependent Transferability Estimation (SDTE)}
The SDTE scenario assumes access to the source data sets where the models have been pre-trained.
Apart from the fact that this assumption is often not met, a drawback of common SDTE metrics, they use distribution matching methods like optimal transport~\cite{otce}, which are typically very expensive to compute. In addition, SDTE metrics are usually not reliable when the discrepancy between the source and target dataset is very high, for example, when comparing entire the ImageNet21K~\cite{imagenet} to Cars~\cite{cars} or Plants~\cite{plantvillage} dataset.

\subsection{Source Independent Transferability Estimation(SITE)}
The Source Independent Transferability Estimation (SITE) assumes access to the source model but not the source training data. This is a more realistic transferability estimation as we might not always have access to the source dataset, nor have the capacity to store the typically very large source datasets like ImageNet~\cite{imagenet} or LAION~\cite{schuhmann2022laionb} in our local setup. SITE methods typically rely on evaluating the feature representation of the source model on the target dataset and its relationship with target labels. 

There are several transferability metrics inspired by various viewpoints. LogME~\cite{logme} formalizes the transferability
estimation as the maximum label marginalized likelihood
and adopts a directed graphical model to solve it. SFDA~\cite{SFDA} proposes a self-challenging mechanism, it first maps the features and then calculates the sum of log-likelihood as the metric.  ETran~\cite{ETran} and PED~\cite{PED} treat the problem of SITE with an energy function, ETran uses energy-based models to detect whether a target dataset is in-distribution or out of distribution for a given pre-trained model whereas PED utilizes potential energy function to modify feature representations to aid other transferability metrics like LogMe and SFDA. NCTI~\cite{NCTI} treats it as a nearest centroid classifier problem and measures how close the geometry of the target features is to their hypothetical state in the
terminal stage of the fine-tuned model. LEEP~\cite{leep} is the average log-likelihood of the log-expected empirical predictor, which is a non-parameter classifier based on the joint distribution of the source and target distribution, N-LEEP~\cite{NLEEP} is a further improvement on LEEP by substituting the output layer with a Gaussian
mixture model. TransRate~\cite{TransRate} treats SITE from an information theory point of view by measuring the transferability as the mutual information between features of target examples extracted by a pre-trained model and their labels. We suggest the survey by \citet{ding2024modeltransfersurveytransferability} for a complete view of transferability metrics. 

Of these existing methods, the approach of NCTI is closest to ours. NCTI checks to which extent the neural collapse criteria~\cite{neuralcollapse} are satisfied on the target embedding. 
We argue (in Section \ref{sec:clustering}) that neural collapse is a byproduct of properties of the loss function, that incentivize in later training stages the embedded points of one class to collapse to their mean. However, if we want to assess the transferability, checking for effects taking place late in training might not be a failproof approach.

\section{Occam's model: Transferability Estimation with finding simpler representation}

\subsection{Problem statement }

We assume that we are given a target dataset $\mathcal{D} = \{(\mathbf{x}_n,y_n)\}_{n=1}^N$ of $N$ labeled points and $M$ pre-trained models $\{\Phi_m=(\phi_m, \psi_m)\}_{m=1}^M$. Each model $\Phi_m$ consists of a feature extractor that returns a $d$-dimensional embedding $\phi_m(x)\in\mathbb{R}^d$ and the final layer or head $\psi_m$ that outputs the label prediction for the given input $x$. The task of estimating transferability is to generate a score for each pre-trained model so that the best model can be identified via a ranking list. For each pre-trained model $\Phi_m$ a transferability metric outputs a scalar score $T_m$ that should be coherent in its ranking with the performance of the fine-tuned classifier $\hat{\Phi}_m$. That is, the goal is to obtain scores $T_m$ such that
\begin{align*}
     T_m&\geq T_n \\\Leftrightarrow  \frac{1}{N}\sum_{n=1}^Np(y_n|\mathbf{x}_n; \hat{\Phi}_m) &\geq \frac{1}{N}\sum_{n=1}^Np(y_n|\mathbf{x}_n; \hat{\Phi}_n),
\end{align*}
where $p(y_n|x_n; \hat{\Phi}_m)$ indicates the probability that the fine-tuned model $\hat{\Phi}_m$ predicts label $y_n$ for input $\mathbf{x}_n$.
A larger $T_m$ indicates better performance model on target data $\mathcal{D}$.

In this work, we aim to asses transferability by finding the degree of simplicity of our representations.
We hypothesize that a classifier can be more easily fine-tuned subject to a target dataset if the embedding $\{\phi(\mathbf{x}_n)\mid 1\leq n\leq N\}$ already has a simple structure in relationship to the labels. As a simple structure, we consider for example an embedding where the points exhibit clustering properties, where the clusters coincide with the classes, or a simple distribution of points in one class. 
Correspondingly, we define our metrics, one that is inspired by clustering properties~\cite{clustering} and one that is motivated by concept characterization and variation~\cite{conceptlearning, conceptvar}.

\subsection{Clustering Approach: Measuring cluster separability as a means of transferability}\label{sec:clustering}
Deep neural network classifiers are trained with the cross-entropy loss. Given a classifier $f_\theta:\mathbb{R}^d\rightarrow [0,1]^C$, indicating the confidence $f_\theta(\mathbf{x})_c$ for class $1\leq c\leq C$ and input $\mathbf{x}$, the cross-entropy loss is defined as
\[CE(y,f_\theta) = -\frac{1}{N}\sum_{n=1}^N\log f_{\theta}(\mathbf{x}_n)_{y_n}.\]
The classifier $f_\theta$ can be seen as a softmax regression (multinomial logistic regression) classifier $h_{W,b}(\mathbf{z})=\mathrm{softmax}(W\mathbf{z}+b)$ applied to an embedding $\phi(\mathbf{x})$. This way, we write $f_\theta(\mathbf{x})=h(\phi(\mathbf{x}))$.

The cross entropy loss is low if there exist vectors $W_{\cdot y}$ and biases $b_y$ for each class $y$, such that the linear function value $W_{\cdot c}^\top \mathbf{z}+b_c$ achieves its maximum value for points from class $y$ at $c=y$. If we want to estimate how well a multinomial logistic regression model would fit the embeddings $\phi(\mathbf{x})$, then we could try to train a multinomial logistic regression on the target embedding. However, this procedure does not take into account that the embedding is flexible. Small changes in the embedding, performed during finetuning, have possibly big impacts on the classifier accuracy. Hence, we rather want to estimate how close the embedding is to a representation that is well-classifiable. 

To gain insights into the classifiability of an embedding, we consider the formulation of the multinomial logistic regression objective as a Linear Discriminant Analysis (LDA) model.
\begin{theorem}\label{thm:loss}
    For any multinomial regression model $h_{W,b}(\mathbf{z})=\mathrm{softmax}(W\mathbf{x}+b)$ exist class centers $\mu_1,\ldots, \mu_C$ such that 
    \[h_{W,b}(\mathbf{x})_y = \frac{\exp(-\frac12\lVert \mathbf{x}-\mu_y\rVert^2)}{\sum_{c=1}^C\exp(-\frac12\lVert \mathbf{x}-\mu_c\rVert^2)}\]
\end{theorem}
The proof can be found in Appendix \ref{proof}.
The theorem shows that we can analyze the classifier's performance also from the viewpoint of a nearest-center classifier. This formulation is close to unsupervised objectives such as $k$-means and it explains the effect of neural collapse~\cite{neuralcollapse}. Once the class boundaries are sufficiently optimized, meaning that the class centers don't change much anymore, the objective still incentivizes the embedded points to be close to its class center. Over time, the class centers hence become centroids and the effects of neural collapse are taking place. 

The cross entropy for point $\mathbf{x}_n$ with label $y_n$ is equal to
\begin{align*}
-\log f_\theta (\mathbf{x}_n)_{y_n} &= 
-\log \frac{\exp(-\frac12\lVert \phi(\mathbf{x}_n)-\mu_{y_n}\rVert^2)}{\sum_{c=1}^C\exp(-\frac12\lVert \mathbf{x}_n-\mu_c\rVert^2)}\\
&= \log \sum_{c\neq y_n}\exp(-\frac12\lVert \phi(\mathbf{x}_n)-\mu_c\rVert^2)\\
&\approx \max_{c\neq y_n} -\frac12\lVert \phi(\mathbf{x}_n)-\mu_c\rVert^2.
\end{align*}
As a result, the loss is minimized if the embedded points of class $y_n$ are far away from the other class centers. Since we do not know the class centers, we propose to evaluate instead the proximity of points between classes, that is equivalent to maximizing the distance to other class centers when the class centers are actually centroids.

\begin{definition} For two classes $a$ and $b$, the \emph{Normalized Interclass Distance} between those classes is given as
\begin{equation}
    D_{ab} = \sum_{y_i =a} \sum_{y_j=b } \frac{\mathrm{dist}(\mathbf{x}_i, \mathbf{x}_j)}{|\{k\mid y_k=a\}| |\{l\mid y_l=b\}|}.
\end{equation}
The \emph{Pairwise Normalized Interclass Distances} on dataset $\mathcal{D}$ with $C$ classes are defined as
\begin{equation}
    \text{INT}(\mathcal{D}) =  \sum_{1\leq a\neq b\leq C} D_{ab}
\end{equation}
\end{definition}
If we choose $\mathrm{dist}(x_i,x_j) = \lVert x_i - x_j\rVert^2$, then the normalized interclass distance is equal to a normalized distance of points to the class-centroid. As outlined before, the class centers are expected to be different from class-centroids. Hence, we focus in particular on using the Euclidean distance $\mathrm{dist}(x_i,x_j) = \lVert x_i - x_j\rVert$, that is less sensitive to outliers.
In Figure \ref{fig:combined} we can see how the $INT$ behaves for a toy dataset for a 2D problem with 3 classes.
 \begin{figure}[H]
    \centering
    \begin{subfigure}[]
        \centering
        \includegraphics[width=0.225\textwidth]{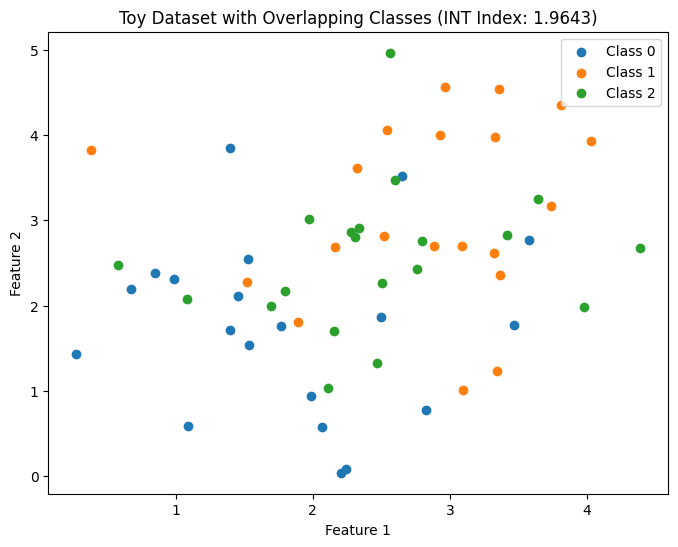}
    \end{subfigure}%
    \begin{subfigure}[]
        \centering
        \includegraphics[width=0.225\textwidth]{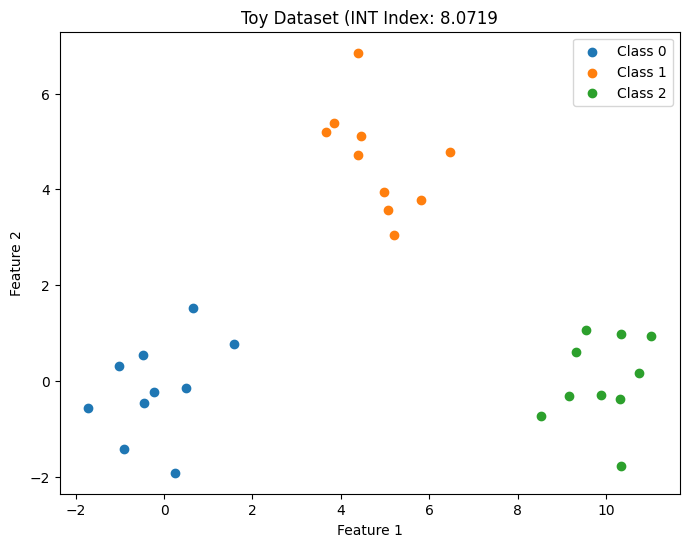}
    \end{subfigure}
    \caption{We show how $INT$ performs with a 2D toy problem. Overlapping classes receive lower INT scores than well-separated clusters.}
    \label{fig:combined}
\end{figure}

\subsection{Concept Variation}

Our second metric is based on concept variation~\cite{conceptvar}, a measure that reflects the irregularity of class label distributions. Understanding this variation helps to assess the structural consistency of a concept $\mathcal{C}$, which can be evaluated by analyzing how class labels are distributed across the feature space.
A highly irregular feature space consists of numerous disjoint regions, often requiring lengthy concept representations. Conversely, a more uniform space contains expansive regions where examples share the same class labels, enabling more concise concept descriptions.
 
The concept variation metric $v$ estimates the likelihood that two neighboring examples belong to distinct classes, thereby approximating the extent of irregularity in the class label distribution. However, the original definition of $v$ is limited to Boolean spaces and assumes that all possible examples in the space are accessible.

Formally, let $\mathbf{x}_{i_1},\ldots, \mathbf{x}_{i_n}$ be the $n$ closest neighbors at Hamming distance one of an example $\mathbf{x}_i$ in an $n$-dimensional Boolean space.  The concept variation for a single instance $\mathbf{x}_i$ is defined as
\begin{equation}
    v(\mathbf{x}_i) = \frac{1}{n}  \sum_{j=1}^n \delta(y_i, y_{i_j})
\end{equation}
where $\delta(y_i, y_{i_j}) = 1$ if $y_i \neq y_{i_j}$, and 0 otherwise. Concept variation is computed as the average of this factor across all examples in the feature space:
\begin{equation}
    v_{total} = \frac{1}{2^n}  \sum^{2^n}_{i=1} v(\mathbf{x}_i) \in [0,1].
\end{equation}
The applicability of $v_{total}$ is limited to artificial domains where all possible examples can be generated. \cite{conceptvar} show how the concepts with high $v_{total}$ are difficult to learn for most conventional inductive algorithms. Since in real-world scenarios this assumption no longer holds, a distance matrix $D$ is used to determine the contribution of each example to the amount of concept variation. The contribution to the amount of concept variation depends on the distance when $\mathbf{x}_i$, $\mathbf{x}_j$ differ in value by using the following function: $w_{ij} = 2^{-\alpha \cdot (D_{ij} / (\sqrt{d} - D_{ij}))}$, where $\alpha$  modulates the effect of distance. The higher the $\alpha$, the less weight is assigned to the distance as it increases between the two examples. The upper bound is reached if $\mathrm{dist}(\mathbf{x}_i,\mathbf{x}_j)=0$ and the lower bound is reached when $\mathrm{dist}(\mathbf{x}_i, \mathbf{x}_j)=\sqrt{n}$ which corresponds to the maximum distance between two examples in an $n$-dimensional feature space. 
We describe our implementation of concept variation in Algorithm \ref{alg:concept_variation}.

\begin{algorithm}[H]
   \caption{Concept Variation}
   \label{alg:concept_variation}
\begin{algorithmic}
\STATE \textbf{Input:} Embedding $\phi({\mathbf{x}})\in\mathbb{R}^d$, Labels $y$, Hyperparameter $\alpha$, 
\STATE \textbf{Output:} Concept variation $v$ of each example

\STATE \textbf{Step 1: Normalize Data:}
\STATE \hspace{1em} Normalize $\phi({\mathbf{x}})$ to $[0, 1]$, such that $\hat{\phi}({\mathbf{x}_i}) \in [0,1]$.

\STATE \textbf{Step 2: Compute Pairwise Euclidean Distances:}
\STATE \hspace{1em} Compute distance matrix $D \in \mathbb{R}^{n \times n}$, where $D_{ij} = \|\hat{\phi}({\mathbf{x}_i}) - \hat{\phi}({\mathbf{x}_j})\|_2$.

\STATE \textbf{Step 3: Compute Concept Variation:}
\STATE \hspace{1em} Compute weight matrix $w \in \mathbb{R}^{n \times n}$, where $\displaystyle w_{ij} = 2^{-\alpha \cdot (D_{ij} / \max\{\sqrt{d} - D_{ij},\epsilon\})}$
\STATE \hspace{1em} Set $w_{ii} = 0$ for all $i$ to eliminate self-contributions.
\STATE \hspace{1em} Compute concept variation for each instance:
\STATE \hspace{2em} $v(\mathbf{x}_i) = \frac{\sum_{j=1}^n w_{ij} \cdot \delta(y_i,y_j)}{\sum_{j=1}^n w_{ij}}$.
\end{algorithmic}
\end{algorithm}
\begin{table*}[tb]
\vspace{-0.1in}
\small
\center
\begin{tabular}{lrrrrrrr|r}
\toprule
\textbf{Dataset} & DIBaS & Flowers & Sports & Plants & Textures & Cars & RESISC & Average \\
\textbf{Metric} &  &  &  &  &  &  &  &  \\
\midrule
LogMe & -0.49 & -0.45 & -0.45 & 0.09 & -0.26 & -0.40 & -0.37 & -0.33 \\
SFDA & -0.09 & -0.17 & -0.17 & 0.21 & -0.26 & -0.07 & -0.17 & -0.11 \\
LDA & 0.59 & 0.09 & -0.14 & 0.27 & 0.07 & -0.25 & 0.09 & 0.10 \\
N-LEEP & -0.22 & -0.53 & -0.60 & -0.15 & -0.10 & -0.60 & -0.53 & -0.39 \\
ETran & \textbf{0.86} & 0.28 & 0.19 & 0.46 & 0.57 & 0.10 & 0.09 & 0.37 \\
TransRate & 0.34 & 0.31 & 0.31 & 0.32 & 0.84 & 0.18 & 0.38 & 0.38 \\
NCTI & 0.14 & -0.37 & -0.37 & -0.07 & -0.51 & -0.26 & -0.53 & -0.28 \\
\midrule
Concept Variance(ours) & 0.03 & 0.52 & 0.52 & 0.13 & 0.04 & -0.06 & 0.93 & 0.30 \\
$INT$(ours) & 0.62 & 0.77 & 0.77 & 0.41 & \textbf{0.87} & \textbf{0.61} & 0.84 & \textbf{0.70} \\
Concept Variance+INT(ours) & 0.56 & \textbf{0.93} &\textbf{ 0.93} & \textbf{0.83} & 0.04 & 0.52 & \textbf{0.93} & 0.68 \\

\bottomrule
\end{tabular}
\caption{Weighted Kendall's $\tau_\omega$ of LogMe SFDA LDA N-LEEP, ETran, TransRate and NCTI vs $INT$ and Concept Variation (the higher the better) in limited data settings.}
\label{tab:extendedresults}
\end{table*}
We use the standard deviation of $v$ for transferability estimation:
\begin{equation}
\sigma_v = \sqrt{\frac{1}{n} \sum_{i=1}^n (v(\mathbf{x}_i) - \bar{v})^2};\quad  \bar{v} = \frac{1}{n} \sum_{i=1}^n v(\mathbf{x}_i)
\end{equation}
A higher standard deviation in concept variation indicates greater diversity in how individual examples’ labels differ from their neighbors across the dataset. In other words, some points may be surrounded by mostly similar labels (low concept variation) while others see very mixed labels (high concept variation), making the overall distribution of concept variation more spread out.

\section{Experiments and Results}

In this section, we evaluate our metrics against the current state-of-the-art SITE metrics in image classification (Section \ref{sec:imgcls}), image classification with a low data regime(Section \ref{ltd_data}), self-supervised learning (Section \ref{sec:ssl}), source selection (Section \ref{sec: source_selection}) and larger network size (Section \ref{sec:big}). We also provide ablation studies in Section \ref{sec:alpha} and Section \ref{sec: intdist}. 

\paragraph{Fine-Tuning Implementation details} 

We can obtain the ground-truth ranking by fine-tuning all pre-trained models with hyper-parameters sweeping on target datasets.
We use a similar fine-tuning setup as SFDA~\cite{SFDA} for our experiments. To obtain test accuracies, we fine-tune pre-trained models with a grid
search over learning rates $\{ 10^{-1}, 10^{-2}, 10^{-3},  10^{-4}\}$ and a weight decay in $\{ 10^{-3},  10^{-4},  10^{-5},  10^{-6}, 0\}$ with early stopping. We determine the best hyper-parameters based on the validation set, and fine-tune the pre-trained model on the target dataset with this parameter and without early stopping. The resulting test accuracy is used as the ground truth score $G_m$ for model $\Phi_m$. This way, we obtain a set of scores $\{G_m\}_{m=1}^M$ as the ground truth to evaluate our pre-trained model rankings.
To compute our metrics, we first perform a single forward pass of the pre-trained model through all target examples to extract their features. We compute the interclass distances with the Euclidean distance metric and we use the default value of $\alpha=2$ as defined by \cite{conceptvar} for our second metric.
We provide the implementation of our metrics in Appendix \ref{lst:intcode}, together with the wall clock time analysis (Experiment \ref{sec:wall}).

Following the previous works~\cite{SFDA,logme,NLEEP} we use weighted Kendall's tau $\tau$~\cite{10.1145/2736277.2741088}  to evaluate the effectiveness of transferability metrics. Kendall's tau $\tau$ returns the ratio of concordant pairs minus discordant pairs when ennumerating all $(_2^M)$ pairs of $\{T_m\}_{m=1}^M$ and  $\{G_m\}_{m=1}^M$ as given by:
\begin{equation}
    \tau = \frac{2}{M(M-1)}\sum_{1\leq i \leq j M} \mathrm{sgn}(G_i-G_j)\mathrm{sgn}(T_i-T_j)
\end{equation}

Where $\mathrm{sgn}(x)$is the signum function returning 1 if $x>0$ and -1 otherwise. In the weighted version of Kendall's tau $\tau_\omega$, the ranking performance of top-performing models is measured to evaluate transferability metrics. In principle, a higher $\tau_\omega$ indicates that the transferability metric produces a better ranking for pretrained models. 

\subsection{Experiment 1: Image Classification}\label{sec:imgcls}
We propose a new and challenging experimental setup for the transferability assessment of neural networks. We use the following datasets for our experiments from the Meta-Album suite~\cite{meta-album-2022}: Flowers~\cite{flowers}, Plant Village~\cite{plantvillage}, DIBaS~\cite{DIBAS}, RESISC~\cite{RESISC}, Cars~\cite{cars}, Textures~\cite{textures}, and 100-sports~\cite{100-sports}. For our pre-trained model zoo we use Data-efficient Image Transformer (DeiT)~\cite{deit}, Co-scale Conv-Attentional Image Transformer (CoaT)~\cite{coat}, Multi-Axis Vision Transformer(MaxViT)~\cite{maxvit}, MobileViT~\cite{mobilevit}, Multi-scale Vision Transformer(MVit)~\cite{mvit}, and  Cross-Covariance Image Transformer(XCiT)~\cite{XCit} pretrained on ImageNet-1k from Huggingface timm library~\cite{rw2019timm}. For image classification tasks, we select LogMe, SFDA, N-LEEP, ETran, LDA Baseline (LDA was available with ETran codebase), NCTI, and TransRate. We describe our model zoo and dataset selection, along with the limitations of the current experimental design, in detail in Appendix \ref{experimental_design}.\\
The results in Table \ref{tab:extendedresults} show the effectiveness of our methods for datasets with $INT$ achieving the highest average $\tau_\omega$. $INT$ seems to outperform on every dataset except DIBaS. 


\subsection{Experiment 2: Limited data setting}\label{ltd_data}
In a realistic transfer learning setting, we have only a small amount of data available for the target task. To emulate this more challenging setting, we perform an experiment with a limited set of 40 examples per class, reported in Table \ref{tab:limited results}. Here as well, we observe that Concept Variance performs much more robustly than other baseline methods and also with respect to $INT$ in the DIBaS and Flowers dataset. The combination of both $INT$ and Concept Variation where the classes are balanced with a limited set of examples. We observe that $INT$ obtains the highest score as well in this scenario and Concept Variance obtains the second highest score in this scenario. We also observe that the sum of $INT$ with Concept Variance obtains higher $\tau_\omega$ than $INT$ or Concept Variance separately.

\begin{table*}[t]
\centering
\begin{tabular}{lrrrrrrr|r}
\toprule
\textbf{Dataset} & DIBaS & Flowers & Sports & Plants & Textures & Cars & RESISC & Average \\
\textbf{Metric} &  &  &  &  &  &  &  &  \\
\midrule
LogMe & 0.42 & -0.25 & -0.53 & -0.07 & -0.02 & -0.53 & -0.37 & -0.19 \\
SFDA & -0.41 & 0.27 & 0.05 & -0.01 & 0.06 & 0.19 & 0.12 & 0.04 \\
LDA & -0.25 & -0.08 & -0.29 & -0.02 & 0.52 & -0.15 & -0.12 & -0.06 \\
N-LEEP & 0.21 & -0.37 & -0.53 & -0.47 & 0.13 & -0.60 & -0.45 & -0.30 \\
ETran & 0.05 & 0.06 & 0.21 & 0.16 & 0.80 & 0.21 & 0.18 & 0.24 \\
TransRate & -0.24 & 0.12 & 0.31 & 0.27 & 0.41 & 0.31 & 0.40 & 0.23 \\
NCTI & 0.15 & -0.53 & -0.53 & 0.02 & -0.12 & -0.53 & -0.36 & -0.27 \\
\midrule
Concept Variance(ours) & \textbf{0.87 }& \textbf{0.73} & 0.27 & 0.17 & -0.10 & 0.28 & 0.73 & 0.42 \\
$INT$(ours) & -0.24 & 0.52 &\textbf{ 0.77} & 0.66 & \textbf{0.62} & \textbf{0.77 }& \textbf{0.91} & 0.57 \\
Concept Variance+INT(ours) & 0.62 & 0.62 & 0.73 & \textbf{0.79} & 0.39 & 0.70 & 0.73 & \textbf{0.65} \\

\bottomrule
\end{tabular}

    \caption{Weighted Kendall's $\tau_\omega$ of LogMe SFDA LDA N-LEEP, ETran, TransRate, and NCTI vs $INT$ and Concept Variation (the higher the better) in the limited data setting.}
    \label{tab:limited results}
\end{table*}

\subsection{Experiment 3: SSL Experiments}\label{sec:ssl}
For our third set of experiments, we apply our metrics to the Self-Supervised Learning(SSL) task. For self-supervised learning we use BYOL~\cite{BYOL}, Deepcluster-v2~\cite{deepclsuter}, Infomin~\cite{InfoMin}, InDis~\cite{InDis}, MoCo-v1~\cite{MoCo}, MoCo-v2~\cite{MoCo}, PCL-v1, PCL-v2~\cite{pclv1v2}, Sela-V2~\cite{selav1v2} and SWAV~\cite{swav} with a pretrained ResNet-50 backbone. We use CIFAR10, CIFAR100~\cite{cifar10and100} and Caltech101~\cite{caltech101} for our experiments.\footnote{We use setup and code from the SFDA~\cite{SFDA} Github repository that is in turn adapted from \citet{ssltransfer} and report the performance in Table \ref{tab:ssl}.} We report the $\tau_\omega$ of LogMe, SFDA and our metrics in Table \ref{tab:ssl}. In SSL experiments we show that $INT$ works on par with SFDA with a tiny difference of average $\tau_\omega$ of 0.752 whereas SFDA average $\tau_\omega$ is 0.749. We did not get result on TransRate for the other two datasets after 120 minutes of compute for the other two datasets.
\begin{table}[H]
\centering
\begin{tabular}{lrrrr}
\toprule
Dataset & Caltech101 & CIFAR10 & CIFAR100 & Avg \\
Metric &  &  &  &  \\
\midrule
LogMe & 0.55 & 0.42 & 0.15 & 0.379 \\
SFDA & 0.61 & 0.85 & 0.79 & \textbf{0.749} \\
TransRate & - & 0.77 & - & - \\
\midrule
CV(ours) & -0.36 & 0.24 & -0.26 & -0.120 \\
$INT$(ours) & 0.67 & 0.83 & 0.76 & \textbf{0.752}\\
\bottomrule
\end{tabular}
\caption{Experiments on SSL methods comparing $INT$, Concept Variance(CV), LogMe and SFDA}
\label{tab:ssl}
\end{table}

\subsection{Experiment 4: Source Selection}\label{sec: source_selection}
In this experiment, we evaluate whether our proposed transferability estimate works well if our source models are trained on more specific source datasets (e.g., where all classes belong to one domain), as opposed to  very broad source datasets, such as ImageNet. 
We use the models CoaT,  DeiT, MAXVit,  MVitv2, and XciT pre-trained on ImageNet-1k and then fine-tuned on the following datasets: Flowers, RESISC, DIBaS, and Plant Village. Our resulting model zoo consists of 24 models. The target datasets in this experiment are Sports, Textures, and Cars. We report the performance for this experiment in Table \ref{tab:sourceselection}. We observe that $INT$ performs here together with TransRate best. 
\begin{table}[H]
\scriptsize
\setlength{\tabcolsep}{12pt} 
\centering
\begin{tabular}{lrrrr}
\toprule
\textbf{Dataset}& Sports & Textures & Cars & Average \\
\textbf{Metric} &  &  &  &  \\
\midrule
LogMe & 0.384 & 0.477 & 0.325 & 0.395 \\

TransRate & 0.433 & 0.158 & 0.800 & \textbf{0.464} \\
\midrule
$INT$(ours) & 0.441 & 0.223 & 0.723 & \textbf{0.464} \\

\bottomrule
\end{tabular}
    \caption{Source selection results comparing $INT$, TransRate and LogMe}
    \label{tab:sourceselection}
\end{table}

\subsection{Experiment 5: Bigger Networks}\label{sec:big}
In this experiment, we evaluate bigger models with larger embedding sizes as well as parameter range. We take the larger version of models introduced in Section \ref{sec:imgcls}, details of selected models can be found in Appendix \ref{models}. The embedding size also quadruples in most of the models. We report our findings for this experiment in Table \ref{tab: Big datasets}. We observe every metric performance drop in this scenario. This also implies that it is harder to estimate the representational complexity when embeddings are larger. Our methods still show much better performance than other metrics. The performance is notably worse for the DIBaS and Cars datasets, where every metric yields a negative $\tau_\omega$. We found that this decline is primarily due to two networks that performed exceptionally poorly, achieving accuracies between 10-35\%—the lowest observed across all networks and datasets.

\begin{table*}[t]
    \centering
\begin{tabular}{lrrrrrrrr}
\toprule
Dataset & DIBaS & Flowers & Sports & Plants & Textures & Cars & RESISC & Average \\
Metric &  &  &  &  &  &  &  &  \\
\midrule
LogMe & -0.33 & -0.90 & -0.53 & -0.05 & 0.53 & -0.49 & -0.23 & -0.29 \\
SFDA & -0.22 & -0.05 & 0.25 & \textbf{0.84 }& 0.44 & -0.49 & \textbf{0.19} & 0.14 \\
LDA & -0.14 & -0.25 & -0.06 & 0.73 & 0.63 & -0.75 & -0.22 & -0.01 \\
ETran & -0.53 & -0.30 & -0.43 & 0.72 & 0.52 & -0.79 & -0.27 & -0.16 \\
TransRate & -0.21 & -0.35 & -0.13 & 0.69 & 0.63 & -0.61 & -0.14 & -0.02 \\
NCTI & 0.02 & -0.59 & -0.09 & -0.25 & 0.08 & -0.36 & -0.18 & -0.19 \\
\midrule
$INT$ & -0.21 & 0.40 & 0.41 & 0.80 & \textbf{0.63} & -0.49 & 0.05 & \textbf{0.23} \\
Concept Variation & \textbf{0.90} & \textbf{0.58} & \textbf{0.46} & -0.15 & 0.13 & -0.05 & -0.53 & \textbf{0.19} \\
\bottomrule
\end{tabular}

    \caption{Experiment on Larger Networks}
    \label{tab: Big datasets}
\end{table*}

\subsection{Experiment 6: Effect of $\alpha$ on Concept Variance}\label{sec:alpha}
We analyze the effect of class weights on Concept Variance performance. We show the effect of $\alpha$ in Figure \ref{fig: alpha}. The results suggest that the optimal $\alpha$ value is between 2-5.
\begin{figure} 
    \includegraphics[width=\linewidth]{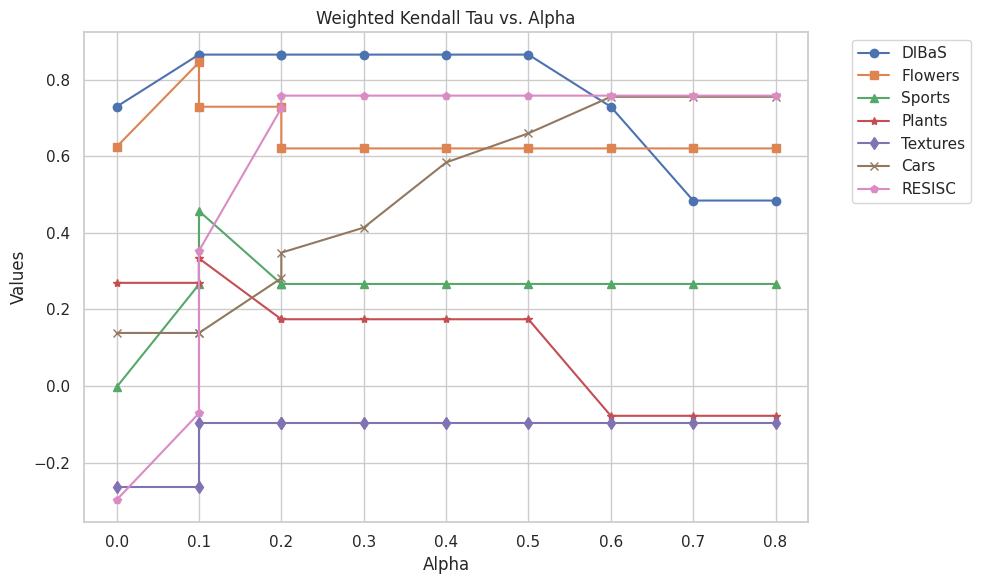}
    \caption{Alpha $\alpha$ studies}\label{fig: alpha}
\end{figure}

\subsection{Experiment 7: $INT$ distance metric studies}\label{sec: intdist}
In this study, we analyze how $INT$ performs under various distance metrics. We compare the following metrics: Euclidean, Squared Euclidean, Manhattan, and Cosine. We used the implementations of the pairwise distances from the CuML~\cite{cuml} and CuPy~\cite{cupy_learningsys2017} libraries. We report the performance of $INT$ with different distance metrics in Table \ref{tab:int_ablation}. We observe that Euclidean gives the most optimal performance out of all very close to squared Euclidean with the difference in the Textures dataset. Cosine distance reports the worst performance with negative $\tau_\omega$. 

\begin{table*}[h]
    \centering
\begin{tabular}{lrrrrrrrr}
\toprule
Dataset & DIBaS & Flowers & Sports & Plants & Textures & Cars & RESISC & Average \\
Distance &  &  &  &  &  &  &  &  \\
\midrule
Euclidean & 0.62 & 0.77 & 0.77 & 0.41 & 0.87 & 0.61 & 0.84 & 0.70 \\
Squared Euclidean & 0.62 & 0.77 & 0.77 & 0.41 & 0.79 & 0.61 & 0.84 & 0.69 \\
Manhattan & 0.61 & 0.62 & 0.62 & 0.54 & 0.77 & 0.50 & 0.70 & 0.62 \\
Cosine & 0.09 & -0.35 & -0.35 & -0.28 & 0.43 & -0.50 & -0.16 & -0.16 \\
\bottomrule
\end{tabular}
    \caption{Extended $INT$ Ablation study}
    \label{tab:int_ablation}
\end{table*}

\subsection{Experiment 8: Wall clock analysis/time taken}\label{sec:wall}
We benchmark the computation time for scoring MVitv2 Base on the DIBaS dataset. We observe that our metrics provide better performance than second and third-best performing metrics(Transrate and SFDA) in a fraction of time with $INT$ performing 5 times faster than SFDA and 43 times better than TransRate and Concept Variance performing 30 times faster than SFDA and 280 times faster than TransRate, making these metrics suitable for large scale assessment of pretrained models. 
\begin{table}[H]
    \centering
\begin{tabular}{l|r}
\toprule
Metric & Time(s) \\
\midrule
LogMe & 0.43\\
SFDA & 6.32 \\
TransRate & 56.35 \\
\midrule
INT(ours) & \textbf{1.28} \\
Concept Variance(ours) & \textbf{0.20} \\
\bottomrule
\end{tabular}
    \caption{Wall clock benchmarking of different top-performing metrics (in seconds).}
    \label{tab:time}
\end{table}

\section{Conclusion and Discussion}
With the current proliferation of pretrained models, finding the best model for a given target task becomes an increasingly important component of transfer learning. In this work, we propose two novel transferability metrics for classification tasks that are based on the idea that the 'simplest' representation, i.e. the one that leads to the clearest separation of classes, will be the best starting point for future finetuning. We evaluate both metrics, Pairwise Normalized Interclass Distance ($INT$) and Concept Variance, on a wide set of in-depth experiments across many image classification problems. We find that the $INT$ metric significantly outperforms all of the state-of-the-art transferability metrics (by 38\% in Experiment \ref{sec:imgcls} and 33\% in Experiment \ref{ltd_data}), although depending on the target task, Concept Variance or a combination of both works better. Moreover, both metrics can be computed efficiently, faster or on par with existing methods.

\subsection{Future work}
An important direction for future work is extending transferability metrics to tasks beyond classification, such as object detection, keypoint regression, semantic segmentation, and depth estimation. These tasks pose unique challenges, as they require processing continuous and multidimensional target variables rather than discrete labels. Nevertheless, we believe that metrics that are similarly based on the complexity of the pre-trained embeddings can provide valuable insights and improvements for pretrained model selection and transfer learning in these domains as well.


\section*{Impact Statement}
This paper presents the work around transferability estimation and its links to simplicity of representations. Our work can help in reducing the time taken by ML practitioners to select a pretrained model hence also reducing computational costs of training multiple models and reducing carbon footprint of a ML workflow.
\bibliography{main}

\begin{thebibliography}{50}
\providecommand{\natexlab}[1]{#1}
\providecommand{\url}[1]{\texttt{#1}}
\expandafter\ifx\csname urlstyle\endcsname\relax
  \providecommand{\doi}[1]{doi: #1}\else
  \providecommand{\doi}{doi: \begingroup \urlstyle{rm}\Url}\fi

\bibitem[Bao et~al.(2019)Bao, Li, Huang, Zhang, Zheng, Zamir, and Guibas]{Hscore}
Bao, Y., Li, Y., Huang, S.-L., Zhang, L., Zheng, L., Zamir, A., and Guibas, L.
\newblock An information-theoretic approach to transferability in task transfer learning.
\newblock In \emph{2019 IEEE International Conference on Image Processing (ICIP)}, pp.\  2309--2313, 2019.
\newblock \doi{10.1109/ICIP.2019.8803726}.

\bibitem[Bezdek \& Pal(1998)Bezdek and Pal]{clustering}
Bezdek, J. and Pal, N.
\newblock Some new indexes of cluster validity.
\newblock \emph{IEEE Transactions on Systems, Man, and Cybernetics, Part B (Cybernetics)}, 28\penalty0 (3):\penalty0 301--315, 1998.
\newblock \doi{10.1109/3477.678624}.

\bibitem[Caron et~al.(2018)Caron, Bojanowski, Joulin, and Douze]{deepclsuter}
Caron, M., Bojanowski, P., Joulin, A., and Douze, M.
\newblock Deep clustering for unsupervised learning of visual features.
\newblock In \emph{European Conference on Computer Vision}, 2018.

\bibitem[Caron et~al.(2020)Caron, Misra, Mairal, Goyal, Bojanowski, and Joulin]{swav}
Caron, M., Misra, I., Mairal, J., Goyal, P., Bojanowski, P., and Joulin, A.
\newblock Unsupervised learning of visual features by contrasting cluster assignments.
\newblock In \emph{Advances in Neural Information Processing Systems}. Curran Associates, Inc., 2020.

\bibitem[Cheng et~al.(2017)Cheng, Han, and Lu]{RESISC}
Cheng, G., Han, J., and Lu, X.
\newblock Remote sensing image scene classification: Benchmark and state of the art.
\newblock \emph{CoRR}, abs/1703.00121, 2017.
\newblock URL \url{http://arxiv.org/abs/1703.00121}.

\bibitem[Deng et~al.(2009)Deng, Dong, Socher, Li, Li, and Fei-Fei]{imagenet}
Deng, J., Dong, W., Socher, R., Li, L.-J., Li, K., and Fei-Fei, L.
\newblock Imagenet: A large-scale hierarchical image database.
\newblock In \emph{2009 IEEE Conference on Computer Vision and Pattern Recognition}, pp.\  248--255, 2009.
\newblock \doi{10.1109/CVPR.2009.5206848}.

\bibitem[Ding et~al.(2024)Ding, Jiang, Yu, Zheng, and Liang]{ding2024modeltransfersurveytransferability}
Ding, Y., Jiang, B., Yu, A., Zheng, A., and Liang, J.
\newblock Which model to transfer? a survey on transferability estimation, 2024.
\newblock URL \url{https://arxiv.org/abs/2402.15231}.

\bibitem[El-Nouby et~al.(2021)El-Nouby, Touvron, Caron, Bojanowski, Douze, Joulin, Laptev, Neverova, Synnaeve, Verbeek, and J{\'e}gou]{XCit}
El-Nouby, A., Touvron, H., Caron, M., Bojanowski, P., Douze, M., Joulin, A., Laptev, I., Neverova, N., Synnaeve, G., Verbeek, J., and J{\'e}gou, H.
\newblock Xcit: Cross-covariance image transformers.
\newblock In \emph{Neural Information Processing Systems}, 2021.
\newblock URL \url{https://api.semanticscholar.org/CorpusID:235458262}.

\bibitem[Ericsson et~al.(2021)Ericsson, Gouk, and Hospedales]{ssltransfer}
Ericsson, L., Gouk, H., and Hospedales, T.~M.
\newblock {How Well Do Self-Supervised Models Transfer?}
\newblock In \emph{CVPR}, 2021.
\newblock URL \url{http://arxiv.org/abs/2011.13377}.

\bibitem[Fan et~al.(2021)Fan, Xiong, Mangalam, Li, Yan, Malik, and Feichtenhofer]{mvit}
Fan, H., Xiong, B., Mangalam, K., Li, Y., Yan, Z., Malik, J., and Feichtenhofer, C.
\newblock Multiscale vision transformers.
\newblock \emph{2021 IEEE/CVF International Conference on Computer Vision (ICCV)}, pp.\  6804--6815, 2021.
\newblock URL \url{https://api.semanticscholar.org/CorpusID:233346705}.

\bibitem[Fei-Fei et~al.(2007)Fei-Fei, Fergus, and Perona]{caltech101}
Fei-Fei, L., Fergus, R., and Perona, P.
\newblock Learning generative visual models from few training examples: An incremental bayesian approach tested on 101 object categories.
\newblock \emph{Computer Vision and Image Understanding}, 2007.
\newblock \doi{https://doi.org/10.1016/j.cviu.2005.09.012}.
\newblock Special issue on Generative Model Based Vision.

\bibitem[Fritz et~al.(2004)Fritz, Hayman, Caputo, and Eklundh]{textures}
Fritz, M., Hayman, E., Caputo, B., and Eklundh, J.
\newblock The kth-tips database.
\newblock 2004.
\newblock URL \url{https://www.csc.kth.se/cvap/databases/kth-tips/index.html}.

\bibitem[G. \& J.(2019)G. and J.]{plantvillage}
G., G. and J., A.~P.
\newblock Identification of plant leaf diseases using a nine-layer deep convolutional neural network.
\newblock \emph{Computers and Electrical Engineering}, 76:\penalty0 323--338, 2019.
\newblock ISSN 0045-7906.
\newblock \doi{https://doi.org/10.1016/j.compeleceng.2019.04.011}.
\newblock URL \url{https://www.sciencedirect.com/science/article/pii/S0045790619300023}.

\bibitem[Gholami et~al.(2023)Gholami, Akbari, Wang, Kamranian, and Zhang]{ETran}
Gholami, M., Akbari, M., Wang, X., Kamranian, B., and Zhang, Y.
\newblock Etran: Energy-based transferability estimation.
\newblock In \emph{Proceedings of the IEEE/CVF International Conference on Computer Vision}, pp.\  18613--18622, 2023.

\bibitem[Grill et~al.(2020)Grill, Strub, Altch\'{e}, Tallec, Richemond, Buchatskaya, Doersch, Avila~Pires, Guo, Gheshlaghi~Azar, Piot, kavukcuoglu, Munos, and Valko]{BYOL}
Grill, J.-B., Strub, F., Altch\'{e}, F., Tallec, C., Richemond, P., Buchatskaya, E., Doersch, C., Avila~Pires, B., Guo, Z., Gheshlaghi~Azar, M., Piot, B., kavukcuoglu, k., Munos, R., and Valko, M.
\newblock Bootstrap your own latent - a new approach to self-supervised learning.
\newblock In \emph{Advances in Neural Information Processing Systems}, volume~33, 2020.

\bibitem[He et~al.(2020)He, Fan, Wu, Xie, and Girshick]{MoCo}
He, K., Fan, H., Wu, Y., Xie, S., and Girshick, R.
\newblock Momentum contrast for unsupervised visual representation learning.
\newblock In \emph{2020 IEEE/CVF Conference on Computer Vision and Pattern Recognition (CVPR)}, 2020.
\newblock \doi{10.1109/CVPR42600.2020.00975}.

\bibitem[Hess et~al.(2020)Hess, Duivesteijn, and Mocanu]{hess2020softmax}
Hess, S., Duivesteijn, W., and Mocanu, D.
\newblock Softmax-based classification is k-means clustering: Formal proof, consequences for adversarial attacks, and improvement through centroid based tailoring.
\newblock \emph{arXiv preprint arXiv:2001.01987}, 2020.

\bibitem[Ho \& Basu(2002)Ho and Basu]{hobasu}
Ho, T.~K. and Basu, M.
\newblock Complexity measures of supervised classification problems.
\newblock \emph{IEEE Transactions on Pattern Analysis and Machine Intelligence}, 24\penalty0 (3):\penalty0 289--300, 2002.
\newblock \doi{10.1109/34.990132}.

\bibitem[Huang et~al.(2022)Huang, Huang, Rong, Yang, and Wei]{TransRate}
Huang, L.-K., Huang, J., Rong, Y., Yang, Q., and Wei, Y.
\newblock Frustratingly easy transferability estimation.
\newblock In \emph{International Conference on Machine Learning}, pp.\  9201--9225. PMLR, 2022.

\bibitem[Kam~Ho(2022)]{repcomplexity}
Kam~Ho, T.
\newblock { Complexity of Representations in Deep Learning }.
\newblock In \emph{2022 26th International Conference on Pattern Recognition (ICPR)}, pp.\  2657--2663, Los Alamitos, CA, USA, August 2022. IEEE Computer Society.
\newblock \doi{10.1109/ICPR56361.2022.9956594}.
\newblock URL \url{https://doi.ieeecomputersociety.org/10.1109/ICPR56361.2022.9956594}.

\bibitem[Krause et~al.(2013)Krause, Stark, Deng, and Fei-Fei]{cars}
Krause, J., Stark, M., Deng, J., and Fei-Fei, L.
\newblock 3d object representations for fine-grained categorization.
\newblock In \emph{4th International IEEE Workshop on 3D Representation and Recognition (3dRR-13)}, Sydney, Australia, 2013.

\bibitem[Krizhevsky et~al.(2009)Krizhevsky, Hinton, et~al.]{cifar10and100}
Krizhevsky, A., Hinton, G., et~al.
\newblock Learning multiple layers of features from tiny images.
\newblock 2009.

\bibitem[Li et~al.(2021{\natexlab{a}})Li, Zhou, Xiong, and Hoi]{pclv1v2}
Li, J., Zhou, P., Xiong, C., and Hoi, S.
\newblock Prototypical contrastive learning of unsupervised representations.
\newblock In \emph{International Conference on Learning Representations}, 2021{\natexlab{a}}.

\bibitem[Li et~al.(2023)Li, Hu, Ge, Shan, and Duan]{PED}
Li, X., Hu, Z., Ge, Y., Shan, Y., and Duan, L.-Y.
\newblock Exploring model transferability through the lens of potential energy.
\newblock In \emph{2023 IEEE/CVF International Conference on Computer Vision (ICCV)}, 2023.
\newblock \doi{10.1109/ICCV51070.2023.00500}.

\bibitem[Li et~al.(2021{\natexlab{b}})Li, Jia, Sang, Zhu, Green, Wang, and Gong]{NLEEP}
Li, Y., Jia, X., Sang, R., Zhu, Y., Green, B., Wang, L., and Gong, B.
\newblock Ranking neural checkpoints.
\newblock In \emph{2021 IEEE/CVF Conference on Computer Vision and Pattern Recognition (CVPR)}, 2021{\natexlab{b}}.
\newblock \doi{10.1109/CVPR46437.2021.00269}.

\bibitem[Mehta \& Rastegari(2022)Mehta and Rastegari]{mobilevit}
Mehta, S. and Rastegari, M.
\newblock Mobilevit: Light-weight, general-purpose, and mobile-friendly vision transformer.
\newblock In \emph{International Conference on Learning Representations}, 2022.
\newblock URL \url{https://openreview.net/forum?id=vh-0sUt8HlG}.

\bibitem[Nguyen et~al.(2020)Nguyen, Hassner, Seeger, and Archambeau]{leep}
Nguyen, C.~V., Hassner, T., Seeger, M., and Archambeau, C.
\newblock Leep: a new measure to evaluate transferability of learned representations.
\newblock In \emph{Proceedings of the 37th International Conference on Machine Learning}, ICML'20. JMLR.org, 2020.

\bibitem[Nilsback \& Zisserman(2008)Nilsback and Zisserman]{flowers}
Nilsback, M.-E. and Zisserman, A.
\newblock Automated flower classification over a large number of classes.
\newblock In \emph{Indian Conference on Computer Vision, Graphics and Image Processing}, Dec 2008.

\bibitem[Okuta et~al.(2017)Okuta, Unno, Nishino, Hido, and Loomis]{cupy_learningsys2017}
Okuta, R., Unno, Y., Nishino, D., Hido, S., and Loomis, C.
\newblock Cupy: A numpy-compatible library for nvidia gpu calculations.
\newblock In \emph{Proceedings of Workshop on Machine Learning Systems (LearningSys) in The Thirty-first Annual Conference on Neural Information Processing Systems (NIPS)}, 2017.
\newblock URL \url{http://learningsys.org/nips17/assets/papers/paper_16.pdf}.

\bibitem[Papyan et~al.(2020)Papyan, Han, and Donoho]{neuralcollapse}
Papyan, V., Han, X., and Donoho, D.~L.
\newblock Prevalence of neural collapse during the terminal phase of deep learning training.
\newblock \emph{Proceedings of the National Academy of Sciences}, 117\penalty0 (40):\penalty0 24652--24663, 2020.

\bibitem[P\'{e}rez \& Rendell(1996)P\'{e}rez and Rendell]{conceptvar}
P\'{e}rez, E. and Rendell, L.~A.
\newblock Learning despite concept variation by finding structure in attribute-based data.
\newblock In \emph{Proceedings of the Thirteenth International Conference on International Conference on Machine Learning}, ICML'96, pp.\  391–399, San Francisco, CA, USA, 1996. Morgan Kaufmann Publishers Inc.
\newblock ISBN 1558604197.

\bibitem[Piosenka()]{100-sports}
Piosenka, G.
\newblock 100 sports image classification.
\newblock URL \url{https://www.kaggle.com/datasets/gpiosenka/sports-classification}.

\bibitem[Raschka et~al.(2020)Raschka, Patterson, and Nolet]{cuml}
Raschka, S., Patterson, J., and Nolet, C.
\newblock Machine learning in python: Main developments and technology trends in data science, machine learning, and artificial intelligence.
\newblock \emph{arXiv preprint arXiv:2002.04803}, 2020.

\bibitem[Rendell \& Cho(1990)Rendell and Cho]{conceptlearning}
Rendell, L. and Cho, H.
\newblock Empirical learning as a function of concept character.
\newblock \emph{Mach. Learn.}, 5\penalty0 (3):\penalty0 267–298, September 1990.
\newblock ISSN 0885-6125.
\newblock \doi{10.1023/A:1022651406695}.
\newblock URL \url{https://doi-org.dianus.libr.tue.nl/10.1023/A:1022651406695}.

\bibitem[Schuhmann et~al.(2022)Schuhmann, Beaumont, Vencu, Gordon, Wightman, Cherti, Coombes, Katta, Mullis, Wortsman, Schramowski, Kundurthy, Crowson, Schmidt, Kaczmarczyk, and Jitsev]{schuhmann2022laionb}
Schuhmann, C., Beaumont, R., Vencu, R., Gordon, C.~W., Wightman, R., Cherti, M., Coombes, T., Katta, A., Mullis, C., Wortsman, M., Schramowski, P., Kundurthy, S.~R., Crowson, K., Schmidt, L., Kaczmarczyk, R., and Jitsev, J.
\newblock {LAION}-5b: An open large-scale dataset for training next generation image-text models.
\newblock In \emph{Thirty-sixth Conference on Neural Information Processing Systems Datasets and Benchmarks Track}, 2022.
\newblock URL \url{https://openreview.net/forum?id=M3Y74vmsMcY}.

\bibitem[Shao et~al.(2022)Shao, Zhao, Ge, Zhang, Yang, Wang, Shan, and Luo]{SFDA}
Shao, W., Zhao, X., Ge, Y., Zhang, Z., Yang, L., Wang, X., Shan, Y., and Luo, P.
\newblock Not all models are equal: Predicting model transferability in a self-challenging fisher space.
\newblock In \emph{Computer Vision – ECCV 2022: 17th European Conference, Tel Aviv, Israel, October 23–27, 2022, Proceedings, Part XXXIV}, pp.\  286–302, Berlin, Heidelberg, 2022. Springer-Verlag.
\newblock ISBN 978-3-031-19829-8.
\newblock \doi{10.1007/978-3-031-19830-4_17}.
\newblock URL \url{https://doi-org.dianus.libr.tue.nl/10.1007/978-3-031-19830-4_17}.

\bibitem[Tan et~al.(2021)Tan, Li, and Huang]{otce}
Tan, Y., Li, Y., and Huang, S.-L.
\newblock Otce: A transferability metric for cross-domain cross-task representations.
\newblock \emph{2021 IEEE/CVF Conference on Computer Vision and Pattern Recognition (CVPR)}, 2021.

\bibitem[Tian et~al.(2020)Tian, Sun, Poole, Krishnan, Schmid, and Isola]{InfoMin}
Tian, Y., Sun, C., Poole, B., Krishnan, D., Schmid, C., and Isola, P.
\newblock What makes for good views for contrastive learning?
\newblock In \emph{Advances in Neural Information Processing Systems}. Curran Associates, Inc., 2020.

\bibitem[Touvron et~al.(2020)Touvron, Cord, Douze, Massa, Sablayrolles, and J'egou]{deit}
Touvron, H., Cord, M., Douze, M., Massa, F., Sablayrolles, A., and J'egou, H.
\newblock Training data-efficient image transformers \& distillation through attention.
\newblock In \emph{International Conference on Machine Learning}, 2020.

\bibitem[Tran et~al.(2019)Tran, Nguyen, and Hassner]{NCE}
Tran, A., Nguyen, C., and Hassner, T.
\newblock Transferability and hardness of supervised classification tasks.
\newblock In \emph{2019 IEEE/CVF International Conference on Computer Vision (ICCV)}, pp.\  1395--1405, 2019.
\newblock \doi{10.1109/ICCV.2019.00148}.

\bibitem[Tu et~al.(2022)Tu, Talebi, Zhang, Yang, Milanfar, Bovik, and Li]{maxvit}
Tu, Z., Talebi, H., Zhang, H., Yang, F., Milanfar, P., Bovik, A.~C., and Li, Y.
\newblock Maxvit: Multi-axis vision transformer.
\newblock In \emph{European Conference on Computer Vision}, 2022.
\newblock URL \url{https://api.semanticscholar.org/CorpusID:247939839}.

\bibitem[Ullah et~al.(2022)Ullah, Carrion, Escalera, Guyon, Huisman, Mohr, van Rijn, Sun, Vanschoren, and Vu]{meta-album-2022}
Ullah, I., Carrion, D., Escalera, S., Guyon, I.~M., Huisman, M., Mohr, F., van Rijn, J.~N., Sun, H., Vanschoren, J., and Vu, P.~A.
\newblock Meta-album: Multi-domain meta-dataset for few-shot image classification.
\newblock In \emph{Thirty-sixth Conference on Neural Information Processing Systems Datasets and Benchmarks Track}, 2022.
\newblock URL \url{https://meta-album.github.io/}.

\bibitem[Vigna(2015)]{10.1145/2736277.2741088}
Vigna, S.
\newblock A weighted correlation index for rankings with ties.
\newblock In \emph{Proceedings of the 24th International Conference on World Wide Web}, WWW '15, Republic and Canton of Geneva, CHE, 2015. International World Wide Web Conferences Steering Committee.
\newblock \doi{10.1145/2736277.2741088}.

\bibitem[Wang et~al.(2023)Wang, Luo, Zheng, Huang, and Baktashmotlagh]{NCTI}
Wang, Z., Luo, Y., Zheng, L., Huang, Z., and Baktashmotlagh, M.
\newblock How far pre-trained models are from neural collapse on the target dataset informs their transferability.
\newblock In \emph{2023 IEEE/CVF International Conference on Computer Vision (ICCV)}, 2023.
\newblock \doi{10.1109/ICCV51070.2023.00511}.

\bibitem[Wightman(2019)]{rw2019timm}
Wightman, R.
\newblock Pytorch image models.
\newblock \url{https://github.com/rwightman/pytorch-image-models}, 2019.

\bibitem[Wu et~al.(2018)Wu, Xiong, Yu, and Lin]{InDis}
Wu, Z., Xiong, Y., Yu, S.~X., and Lin, D.
\newblock Unsupervised feature learning via non-parametric instance discrimination.
\newblock In \emph{2018 IEEE/CVF Conference on Computer Vision and Pattern Recognition}, 2018.
\newblock \doi{10.1109/CVPR.2018.00393}.

\bibitem[Xu et~al.(2021)Xu, Xu, Chang, and Tu]{coat}
Xu, W., Xu, Y., Chang, T.~A., and Tu, Z.
\newblock Co-scale conv-attentional image transformers.
\newblock \emph{2021 IEEE/CVF International Conference on Computer Vision (ICCV)}, pp.\  9961--9970, 2021.
\newblock URL \url{https://api.semanticscholar.org/CorpusID:233219797}.

\bibitem[YM. et~al.(2020)YM., C., and A.]{selav1v2}
YM., A., C., R., and A., V.
\newblock Self-labelling via simultaneous clustering and representation learning.
\newblock In \emph{International Conference on Learning Representations}, 2020.

\bibitem[You et~al.(2021)You, Liu, Long, and Wang]{logme}
You, K., Liu, Y., Long, M., and Wang, J.
\newblock Logme: Practical assessment of pre-trained models for transfer learning.
\newblock In \emph{International Conference on Machine Learning}, 2021.
\newblock URL \url{https://api.semanticscholar.org/CorpusID:231985863}.

\bibitem[Zielinski et~al.(2017)Zielinski, Plichta, Misztal, Spurek, Brzychczy-Wloch, and Ochonska]{DIBAS}
Zielinski, B., Plichta, A., Misztal, K., Spurek, P., Brzychczy-Wloch, M., and Ochonska, D.
\newblock Deep learning approach to bacterial colony classification.
\newblock \emph{PLOS ONE}, 12\penalty0 (9):\penalty0 1--14, 09 2017.
\newblock \doi{10.1371/journal.pone.0184554}.
\newblock URL \url{https://doi.org/10.1371/journal.pone.0184554}.

\end{thebibliography}
\bibliographystyle{icml2025}

\newpage
\appendix
\onecolumn
\section{Proofs}\label{proof}
We state first the following theorem from ~\cite{hess2020softmax} for completeness, since we need it for our analysis.
\begin{theorem}\label{thm:pen=km}
Let $h_{W,b}(\mathbf{x})=\mathrm{softmax}(W\mathbf{x}+b)$ be a multinomial regression model with $W\in\mathbb{R}^{C\times d}$ and $b\in\mathbb{R}^C$, computing class predictions as
$y=\argmax_{1\leq c\leq C} \mathbf{x}^\top W_{\cdot c} +b_c$. If $W$ has at least a rank of $r\geq C$, then there exist $C$ class centers $\mu_c\in\mathbb{R}^d$ such that every point $\mathbf{x}$ is assigned to the class having the nearest center:
\[y=\argmin_{1\leq c\leq C}\left\lVert\mathbf{x}-\mu_{c}\right\rVert^2.\]
\end{theorem}
\begin{proof}
We show that for any dataset and network there exists a set of class centers such that the classification does not change when classifying according to the nearest center.

We gather given data points in the matrix $D$:

\[D^\top=\begin{pmatrix} \mathbf{x}_1 & \hdots & \mathbf{x}_N\end{pmatrix}\in \mathbb{R}^{d\times N}.\]

We define $Z=W+\mathbf{v}\mathbf{1}_C^\top$, where $\mathbf{v}\in \mathbb{R}^d$ and $\mathbf{1}_C\in\{1\}^C$ is a constant one vector. The (soft)max classification of all data points in $D$ is then given by the one-hot encoded matrix $Y\in\{0,1\}^{N\times C} $ that optimizes the objective
\begin{equation}
\argmax_Y \tr(Y(W^\top D^\top +\mathbf{b}\mathbf{1}_{N}^\top))
= \argmin_Y \lVert D-YZ^\top\rVert^2 +\tr((2\mathbf{b}\mathbf{1}_{C}^\top -Z^\top Z) Y^\top Y)) \label{eq:k_means_nearest_centroid}
\end{equation}

The matrix $Z\in \mathbb{R}^{d\times C}$ indicates a set of $C$ centers by its columns. The first term of Equation (\ref{eq:k_means_nearest_centroid}) is minimized if $Y$ assigns the class with the closest centroid to each data point in $D$. Hence, if we can show that there exists a vector $v\in\mathbb{R}^d$ such that the second term of Equation (\ref{eq:k_means_nearest_centroid}) is equal to zero (given $D$ and $W$) then we have shown what we wanted to prove. Since $\lvert Y_{j\cdot}\rvert=1$ (every point is assigned to exactly one class), the matrix $Y^\top Y$ is a diagonal matrix, having the number of data points assigned to each class on the diagonal: $Y^\top Y = \diag(\lvert Y_{\cdot 1}\rvert,\ldots, \lvert Y_{\cdot c}\rvert)$. Hence, the trace term on the right of Equation (\ref{eq:k_means_nearest_centroid}) equals
\begin{equation}
\sum_{c=1}^C (2\mathbf{b}\mathbf{1}_C^\top -Z^\top Z)_{cc} \lvert Y_{\cdot c}\rvert
= \sum_{c=1}^C (2b_c -\lVert W_{\cdot k}\rVert^2 -2v^\top W_{\cdot k}) \lvert Y_{\cdot k}\rvert -\lVert \mathbf{v}\rVert^2 m\label{eq:traceterm}
\end{equation}
We define the vector $\mathbf{u}\in\mathbb{R}^C$ such that $u_c = b_c -\frac{1}{2}\lVert W_{\cdot c}\rVert^2$. The right term of Equation (\ref{eq:traceterm}) is constant for a vector $\mathbf{v}$ satisfying $u_c=\mathbf{v}^\top W_{\cdot c}$ for $1\leq c\leq C$. That is, we need to solve the following equation for $\mathbf{v}$:
\begin{align*}
    \mathbf{u}=W^\top \mathbf{v}= V\Sigma U^\top \mathbf{v}.
\end{align*}
Since the rank of $W$ is $C$ (full column rank), this equation has a solution. It is given by the SVD of $W=U\Sigma V^\top$, where $U\in\mathbb{R}^{d\times C}$ is a left orthogonal matrix ($U^\top U=I$), $\Sigma\in\mathbb{R}_+^{C\times C}$ is a diagonal matrix having only positive values, and $V\in\mathbb{R}^{C\times C}$ is an orthogonal matrix ($V^\top V = VV^\top = I$). Setting $\mathbf{v}=U\Sigma^{-1}V^\top \mathbf{u}$, this vector solves the equation.
\end{proof}
\begin{theorem}[Restatement of Thm~\ref{thm:loss} ]
    For any multinomial regression model $h_{W,b}(\mathbf{z})=\mathrm{softmax}(W\mathbf{x}+b)$ exist class centers $\mu_1,\ldots, \mu_C$ such that 
    \[h_{W,b}(\mathbf{x})_y = \frac{\exp(-\frac12\lVert \mathbf{x}-\mu_y\rVert^2)}{\sum_{c=1}^C\exp(-\frac12\lVert \mathbf{x}-\mu_c\rVert^2)}\]
\end{theorem}
\begin{proof}
According to the proof of Thm. \ref{thm:pen=km}, exists a vector $\mathbf{v}$ such that the class boundaries do not change when we replace the standard multinomial regression prediction with the nearest center prediction, where the centers are defined by $\mu_c=W_{\cdot c}+\mathbf{v}$. For this set of centroids we have the following relationship of the multinomial regression confidence to the center-based confidence.
    \begin{align*}
    \mathrm{softmax}(W\mathbf{x} + \mathbf{b})_y &= \frac{\exp(\mathbf{x}^\top W_{\cdot y} +b_y)}{\sum_{c=1}^C\exp(\mathbf{x}^\top W_{\cdot c}+b_c)}\cdot \frac{\exp(\mathbf{x}^\top \mathbf{v})}{\exp(\mathbf{x}^\top \mathbf{v})}\\
    &=\frac{\exp(\mathbf{x}^\top \mu_y+b_y)}{\sum_{c=1}^C\exp(\mathbf{x}^\top \mu_c+b_c)}\\
    &=\frac{\exp(-\frac12\lVert \mu_y\rVert^2 +\mathbf{x}^\top \mu_y+b_y+\frac12\lVert \mu_y\rVert^2)}{\sum_{c=1}^C\exp(-\frac12\lVert \mu_c\rVert^2 + \mathbf{x}^\top \mu_c+b_c+\frac12\lVert \mu_c\rVert^2)}\frac{\exp(\frac12 \lVert \mathbf{x}\rVert^2)}{\exp(\frac12 \lVert \mathbf{x}\rVert^2)}\\
    &=\frac{\exp(-\frac12\lVert \mu_y- \mathbf{x}\rVert^2)\exp(b_y+\frac12\lVert \mu_y\rVert^2)}{\sum_{c=1}^C\exp(-\frac12\lVert \mu_c- \mathbf{x}\rVert^2)\exp(b_c+\frac12\lVert \mu_c\rVert^2)}
\end{align*}
Since the class boundaries of the nearest-center prediction do not change the original class boundaries, we have for any point $\mathbf{x}$ on the decision boundary between class $a$ and $b$ $\mathrm{softmax}(W\mathbf{x}+\mathbf{b})_a=\mathrm{softmax}(W\mathbf{x}_2+\mathbf{b})$. At the same time we have $\exp(-\frac12\lVert \mu_{y_1} - \mathbf{x}\rVert^2)=\exp(-\frac12\lVert \mu_{y_2} - \mathbf{x}\rVert^2)$ because the nearest center classifier with centers $\mu_c$ is not changing the decision boundary according to Thm~\ref{thm:pen=km}. As a result we have (using the equivalence above)
$$\exp(b_{y_1}+\frac12\lVert \mu_{y_1}\rVert^2) = \exp(b_{y_2}+\frac12\lVert \mu_{y_2}\rVert^2).$$
Since this equation holds for any point on the decision boundary, we have proven our final result.
\end{proof}
\section{Omitted Experiment Details in Section 4}
\subsection{Dataset details}
\begin{enumerate}
    \item \textbf{DIBaS:} Digital Image of Bacterial Species (DIBaS). The Digital Images of Bacteria Species dataset (DIBaS) (https://github.com/gallardorafael/DIBaS-Dataset) is a dataset of 33 bacterial species with around 20 images for each species. 
    \item \textbf{Flowers}: Flowers dataset from Visual Geometry Group, University of Oxford. The Flowers dataset(https://www.robots.ox.ac.uk/~vgg/data/flowers/102/index.html) consists of a variety of flowers gathered from different websites and some are photographed by the original creators. These flowers are commonly found in the UK. The images generally have large scale, pose and light variations. Some categories of flowers in the dataset has large variations of flowers while other have similar flowers in a category. 
    \item \textbf{Sports: }The 100-Sports dataset(https://www.kaggle.com/datasets/gpiosenka/sports-classification) is a collection of sports images covering 73 different sports. Images are 224x224x3 in size and in .jpg format. Images were gathered from internet searches. The images were scanned with a duplicate image detector program and all duplicate images were removed. 
    \item \textbf{Plants: }The Plant Village dataset \url{https://data.mendeley.com/datasets/tywbtsjrjv/1} contains camera photos of 17 crop leaves. The original image resolution is 256x256 px. This collection covers 26 plant diseases and 12 healthy plants.
    \item \textbf{Textures: }The original Textures dataset is a combination of 4 texture datasets: KTH-TIPS and KTH-TIPS 2 (https://www.csc.kth.se/cvap/databases/kth-tips/index.html), Kylberg Textures Dataset (http://www.cb.uu.se/~gustaf/texture/) and UIUC Textures Dataset. The data in all four datasets is collected in laboratory conditions, i.e., images were captured in a controlled environment with configurable brightness, luminosity, scale and angle. The KTH-TIPS dataset was collected by Mario Fritz and KTH-TIPS 2 dataset was collected by P. Mallikarjuna and Alireza Tavakoli Targhi, created in 2004 and 2006 respectively. Both of these datasets were prepared under the supervision of Eric Hayman and Barbara Caputo. The data for Kylberg Textures Dataset and UIUC Textures Dataset data was collected by the original authors of these datasets in September 2010 and August 2005 respectively. 
    \item \textbf{Cars: } The original Cars dataset (https://ai.stanford.edu/~jkrause/cars/car\_dataset.html) was collected in 2013, and it contains more than 16 000 images from 196 classes of cars. Most images are on the road, but some have different backgrounds, and each image has only one car. Each class can have 48 to 136 images of variable resolutions. 
    \item \textbf{RESISC:} RESISC45 dataset(https://gcheng-nwpu.github.io/) gathers 700 RGB images of size 256x256 px for each of 45 scene categories. The data authors strive to provide a challenging dataset by increasing both within-class diversity and between-class similarity, as well as integrating many image variations. Even though RESISC45 does not propose a label hierarchy, it can be created from other common aerial image label organization scheme.
    \item \textbf{CIFAR10:} The dataset contains 60,000 color images in 10 classes, with each image in the size
of 32×32. Each class has 5,000 training samples and 1,000 testing samples.
    \item \textbf{CIFAR 100:} The dataset is the same as CIFAR-10 except that it has 100 classes each of which
contains 500 training images and 100 testing images.
DTD (Cimpoi et al., 2014) The dataset consists of
    \item  \textbf{CALTECH 101:} The dataset contains 9,146 images from 101 object categories. The number of images in each category is between 40 and 800. 
\end{enumerate}

\subsection{Models}\label{models}
Pretrained \textbf{timm} models used in Experiment \ref{sec:imgcls}, Experiment \ref{ltd_data}, Experiment \ref{sec:alpha} and Experiment \ref{sec: intdist}.
\begin{table}[H]
    \centering
    \begin{tabular}{lccc}
        \toprule
        \textbf{Model Name} & \textbf{Parameters (M)} & \textbf{FLOPs (GMACs)} & \textbf{Input Resolution} \\
        \midrule
        CoaT Lite Tiny & 5.7 & 1.6 & 224 x 224 \\
        DeiT Tiny Patch16 224 & 5.7 & 1.3 & 224 x 224 \\
        MaxViT Tiny TF 224 & 30.9 & 5.6 & 224 x 224 \\
        MobileViT XS & 2.3 & 0.7 & 256 x 256 \\
        MViTv2 Tiny & 5.1 & 1.1 & 224 x 224 \\
        XCiT Tiny 12 P8 224 & 6.7 & 4.8 & 224 x 224 \\
        MobileViTv2 100 & 3.6 & 0.9 & 256 x 256 \\
        \bottomrule
    \end{tabular}
    \caption{Details of the selected models from Hugging Face.}
    \label{tab:model_details}
\end{table}

Pretrained timm models used in Section \ref{sec:big}
\begin{table}[H]
    \centering
    \begin{tabular}{lccc}
        \toprule
        \textbf{Model Name} & \textbf{Parameters (M)} & \textbf{FLOPs (GMACs)} & \textbf{Input Resolution} \\
        \midrule
        DeiT Base Patch16 224 & 86.6 & 17.6 & 224 x 224 \\
        MobileViTv2 200 & 6.9 & 1.7 & 256 x 256 \\
        XCiT Small 12 P8 224 & 26.3 & 9.1 & 224 x 224 \\
        MaxViT Small TF 224 & 69.0 & 11.7 & 224 x 224 \\
        MViTv2 Base & 51.5 & 10.2 & 224 x 224 \\
        CoaT Lite Medium 384 & 20.0 & 4.0 & 384 x 384 \\
        \bottomrule
    \end{tabular}
    \caption{Details of the bigger models from Hugging Face.}
    \label{tab:model_details_big}
\end{table}

\subsection{Experimental design}\label{experimental_design}
We use different search space as compared to the one in recent works such as SFDA, TransRate and ETran, We use a more realistic experimental search space, we limit out model sizes from 5-25M parameters for Experiment \ref{sec:imgcls}, \ref{ltd_data}, \ref{sec:big}. We select models from top performing networks from PaperWithCode leaderboard on ImageNet Classification \url{https://paperswithcode.com/sota/image-classification-on-imagenet}. We also select every target domain dataset form a different domain. Our model search space is more complex, older works have used only 3 variants of models in different sizes which can show more models but in the end larger models eg Resnet 150 always outperform Resnet18 under similar training regimes as one can verify from the test scores( We verify this by looking at the same performance metrics used by SFDA, Etran and PED \url{https://github.com/mgholamikn/ETran/blob/main/tw.py},  So one can just use embedding size as a transferability metric and get a positive/better performing $\tau_\omega$.). Our datasets are also larger and have larger image sizes as well(128x128). Our reasoning behind the suggested experiment design is that one usually is not trying to select between a 50M parameter model and a 250M parameter model, overparameterized model from the same family tend to work better but different models under same parameter limit is the right case where one need transferability estimation. We make sure we have no network which consistently perform worse than other networks and avoid models from same families. The model training and fine tuning is performed on single V100 GPU and metric calculation is done on single NVIDIA RTX A6000 GPU.

\section{Time complexity analysis}
\subsection{INT}
$\mathcal{O}(k^2.N^2,D)$ Where $k$ is the number of classes, $N$ is the number of samples and $D$ is the embedding size of the network. Assuming $k<<N$ the overall complexity of $INT$ is $\mathcal{O}(N^2.D)$.
\subsection{Concept Variation}
The overall complexity for concept variation $v$ is $\mathcal{O}(N^2.D)$.

\section{Python Code for $INT$ and concept variance}
\begin{lstlisting}[language=Python]

import cupy as cp
import numpy as np
from itertools import combinations
from cuml.metrics import pairwise_distances

def ft_int_gpu(N, y, dist_metric="euclidean"):

    N_gpu = cp.array(N)
    y_gpu = cp.array(y)

    # Get unique classes
    classes = np.unique(y)  # Use NumPy for unique classes since it's small
    class_num = len(classes)

    if class_num == 1:
        return np.nan

    # Compute pairwise normalized interclass distances
    pairwise_norm_intercls_dist = []
    for id_cls_a, id_cls_b in combinations(range(class_num), 2):
        group_a = N_gpu[y_gpu == classes[id_cls_a]]
        group_b = N_gpu[y_gpu == classes[id_cls_b]]
        dists = pairwise_distances(group_a, group_b, metric=dist_metric)
        pairwise_norm_intercls_dist.append(cp.sum(dists) / dists.size)

    # INT computation
    norm_factor = 2.0 / (class_num * (class_num - 1.0))
    sum_intercls_dist = cp.sum(cp.array(pairwise_norm_intercls_dist))

    return float(sum_intercls_dist * norm_factor)
\end{lstlisting}
\label{lst:intcode}

\begin{lstlisting}[language=Python]
def precompute_concept_dist_gpu(N, concept_dist_metric)
    # Convert numpy array to CuPy array
    N = cp.asarray(N)

    # Normalize N to range [0, 1]
    scaler = cuml.preprocessing.MinMaxScaler(feature_range=(0, 1))
    N = scaler.fit_transform(N)

    # Compute pairwise Euclidean distances using cuML
    concept_distances = cuml.metrics.pairwise_distances(N, metric="euclidean")

    return concept_distances

def ft_conceptvar_gpu(
    N: np.ndarray,
    y: np.ndarray,
    conceptvar_alpha: float = 2.0,
    concept_dist_metric: str = "euclidean",
    concept_minimum: float = 1e-10,
) -> t.Tuple[float, cp.ndarray]:
    # Convert numpy arrays to CuPy arrays

    N = cp.asarray(N)
    y = cp.asarray(y)

    concept_distances = precompute_concept_dist_gpu(N, concept_dist_metric)

    n_col = N.shape[1]

    div = cp.sqrt(cp.array(n_col)) - concept_distances
    div = cp.clip(div, a_min=concept_minimum, a_max=None)  
    weights = cp.power(2, -conceptvar_alpha * (concept_distances / div))
    cp.fill_diagonal(weights, 0.0)

    rep_class_matrix = cp.expand_dims(y, 0).repeat(y.shape[0], axis=0)
    class_diff = (rep_class_matrix.T != rep_class_matrix).astype(cp.float32)
    w2 = weights * class_diff
    conceptvar_by_example = cp.sum(w2, axis=0) / cp.sum(weights, axis=0)
    std_dev = cp.std(conceptvar_by_example).item()

    return std_dev
\end{lstlisting}
\label{lst:cvcode}

\section{Additional Tables}

\begin{table}[H]

\begin{tabular}{lrrrrrrrr}
\toprule
Dataset & DIBaS & Flowers & Sports & Plants & Textures & Cars & RESISC & Average \\
Alpha &  &  &  &  &  &  &  &  \\
\midrule
$\alpha=$0.5 & 0.73 & 0.62 & -0.00 & 0.27 & -0.26 & 0.14 & -0.30 & 0.17 \\
$\alpha=$1.0 & 0.87 &\textbf{ 0.84 }& 0.27 & 0.27 & -0.26 & 0.14 & -0.07 & 0.29 \\
$\alpha=$1.5 & 0.87 & 0.73 & \textbf{0.46} & \textbf{{0.33}} & \textbf{-0.10} & 0.14 & 0.35 & 0.40 \\
$\alpha=$2.0 & 0.87 & 0.73 & 0.27 & 0.17 & -0.10 & 0.28 & 0.73 & 0.42 \\
$\alpha=$2.5 & 0.87 & 0.62 & 0.27 & 0.17 & -0.10 & 0.35 & 0.76 & 0.42 \\
$\alpha=$3.0 & 0.87 & 0.62 & 0.27 & 0.17 & -0.10 & 0.41 & \textbf{0.76} & 0.43 \\
$\alpha=$4.0 & 0.87 & 0.62 & 0.27 & 0.17 & -0.10 & 0.58 & 0.76 & 0.45 \\

$\alpha=$5.0 & 0.87 & 0.62 & 0.27 & 0.17 & -0.10 & 0.66 & 0.76 & \textbf{0.46} \\
$\alpha=$6.0 & 0.73 & 0.62 & 0.27 & -0.08 & -0.10 &\textbf{ 0.76} & 0.76 & 0.42 \\
$\alpha=$7.0 & 0.48 & 0.62 & 0.27 & -0.08 & -0.10 & 0.76 & 0.76 & 0.39 \\
$\alpha=$8.0 & 0.48 & 0.62 & 0.27 & -0.08 & -0.10 & 0.76 & 0.76 & 0.39 \\
\bottomrule
\end{tabular}
\caption{Studying the affect of class weights on $\tau_\omega$ on limited data setting}
\end{table}

\end{document}